\documentclass[runningheads]{llncs}
\usepackage{graphicx}
\usepackage{amsmath}
\usepackage{amsfonts}
\usepackage{amssymb}
\usepackage{mathtools}
\usepackage{multirow}
\usepackage{units}
\usepackage{array}
\usepackage{diagbox}
\usepackage{rotating}
\usepackage{siunitx}
\usepackage{etoolbox}
\usepackage{booktabs}
\usepackage{slashbox}
\usepackage{hhline}
\usepackage{footmisc}
\usepackage{stfloats}
\usepackage{algorithm}
\usepackage{algorithmic}

\newcommand{\ie}{i.\,e. }

\newcolumntype{H}{>{\setbox0=\hbox\bgroup}c<{\egroup}@{}}

\begin{document}

\title{Kernel Graph Convolutional Neural Networks}

\author{Giannis Nikolentzos\inst{1} \and
Polykarpos Meladianos\inst{2} \and
Antoine J.-P. Tixier\inst{1} \and
Konstantinos Skianis\inst{1} \and
Michalis Vazirgiannis\inst{1,2}}

\authorrunning{G. Nikolentzos et al.}

\institute{\'Ecole Polytechnique, France \\
\email{\{nikolentzos,anti5662,kskianis,mvazirg\}@lix.polytechnique.fr} \and
Athens University of Economics and Business, Greece \\
\email{pmeladianos@aueb.gr}}

\maketitle              

\begin{abstract}
Graph kernels have been successfully applied to many graph classification problems.
Typically, a kernel is first designed, and then an SVM classifier is trained based on the features defined implicitly by this kernel. This two-stage approach decouples data representation from learning, which is suboptimal.
On the other hand, Convolutional Neural Networks (CNNs) have the capability to learn their own features directly from the raw data during training. Unfortunately, they cannot handle irregular data such as graphs.
We address this challenge by using graph kernels to embed meaningful local neighborhoods of the graphs in a continuous vector space.
A set of filters is then convolved with these patches, pooled, and the output is then passed to a feedforward network. With limited parameter tuning, our approach outperforms strong baselines on 7 out of 10 benchmark datasets. Code and data are publicly available\footnote{\small{\url{https://github.com/giannisnik/cnn-graph-classification}}}.
\end{abstract}

\section{Introduction}\label{sec:introduction}
Graphs are powerful structures that can be used to model almost any kind of data. Social networks, textual documents, the World Wide Web, chemical compounds, and protein-protein interaction networks, are all examples of data that are commonly represented as graphs. As such, graph classification is a very important task, with numerous significant real-world applications.
However, due to the absence of a unified, standard vector representation of graphs, graph classification cannot be tackled with classical machine learning algorithms.

Kernel methods offer a solution to those cases where instances cannot be readily vectorized. The trick is to define a suitable object-object similarity function (known as a kernel function). Then, the matrix of pairwise similarities can be passed to a kernel-based supervised algorithm such as the Support Vector Machine to perform classification. With properly crafted kernels, this two-step approach was shown to give state-of-the-art results on many datasets \cite{shervashidze2011weisfeiler}, and has become standard and widely used. One major limitation of the graph kernel + SVM approach, though, is that representation and learning are two \textit{independent} steps. In other words, the features are precomputed in separation from the training phase, and are not optimized for the downstream task.

Conversely, Convolutional Neural Networks (CNNs) learn their own features from the raw data during training, to maximize performance on the task at hand. CNNs thus provide a very attractive alternative to the aforementioned two-step approach. However, CNNs are designed to work on regular grids, and thus cannot process graphs.

We propose to address this challenge by extracting patches from each input graph via community detection, and by embedding these patches with graph kernels. The patch vectors are then convolved with the filters of a 1D CNN and pooling is applied. Finally, to perform graph classification, a fully-connected layer with a softmax completes the architecture. We compare our proposed method with state-of-the-art graph kernels and a recently introduced neural architecture on 10 bioinformatics and social network datasets. Results show that our Kernel CNN model is very competitive, and offers in many cases significant accuracy gains.

\section{Related Work}\label{sec:related_work}
\subsubsection{Graph kernels.} A graph kernel is a kernel function defined on pairs of graphs. Graph kernels can be viewed as graph similarity functions, and currently serve as the dominant tool for graph classification. Most graph kernels compute the similarity between two networks by comparing their substructures, which can be specific subgraphs \cite{shervashidze2009efficient}, random walks \cite{vishwanathan2010graph}, cycles \cite{horvath2004cyclic}, or paths \cite{borgwardt2005shortest}, among others.
The Weisfeiler-Lehman framework operates on top of existing kernels and improves their performance by using a relabeling procedure based on the Weisfeiler-Lehman test of isomorphism \cite{shervashidze2011weisfeiler}. 
Recently, two other frameworks were presented for deriving variants of popular graph kernels \cite{yanardag2015deep,yanardag2015structural}. Inspired by recent advances in NLP, they offer a way to take into account substructure similarity.
Some graph kernels not restricted to comparing substructures of graphs but that also capture their global properties have also been proposed. Examples include graph kernels based on the Lov\'asz number and the corresponding orthonormal representation \cite{johansson2014global}, the pyramid match graph kernel that embeds vertices in a feature space and computes an approximate correspondence between them \cite{nikolentzos2017matching}, and the Multiscale Laplacian graph kernel, which captures similarity at different granularity levels by considering a hierarchy of nested subgraphs \cite{kondor2016multiscale}.

\subsubsection{Graph CNNs.} Extending CNNs to graphs has experienced a surge of interest in recent years. A first class of methods use spectral properties of graphs. An early generalization of the convolution operator to graphs was based on the eigenvectors of the Laplacian matrix \cite{bruna2014}. A more efficient model using Chebyshev polynomials approximation to represent the spectral filters was later presented \cite{defferrard2016convolutional}. All of these methods, however, assume a fixed graph structure and are thus not applicable to our setting. The model of \cite{defferrard2016convolutional} was then simplified by using a first-order approximation of the spectral filters \cite{kipf2016semi}, but within the context of a \textit{node} classification problem (which again, differs from our \textit{graph} classification setting). Unlike spectral methods, spatial methods \cite{niepert2016learning,vialatte2016generalizing} operate directly on the topology of the graph. Finally, some other techniques rely on node embeddings obtained as an unsupervised pre-processing step, like \cite{tixier2017classifying}, in which graphs are represented as stacks of bivariate histograms and passed to a classical 2D CNN for images.

The work closest to ours is probably \cite{niepert2016learning}. To extract a set of patches from the input graph, the authors (1) construct an ordered sequence of vertices from the graph, (2) create a neighborhood graph of constant size for each selected vertex, and (3) generate a vector representation (patch) for each neighborhood using graph labeling procedures such that nodes with similar structural roles in the neighborhood graph are positioned similarly in the vector space. The extracted patches are then fed to a 1D CNN. In contrast to the above work, we extract neighborhoods of varying sizes from the graph in a more direct and natural way (via community detection), and use graph kernels to normalize our patches. We present our approach in more details in the next section.

\section{Proposed approach}\label{sec:contribution}
In what follows, we present the main ideas and building blocks of our model. The overarching process flow is illustrated in Figure~\ref{fig:overview}.

\begin{figure}[t]
  \centering
  \includegraphics[width=.9\linewidth]{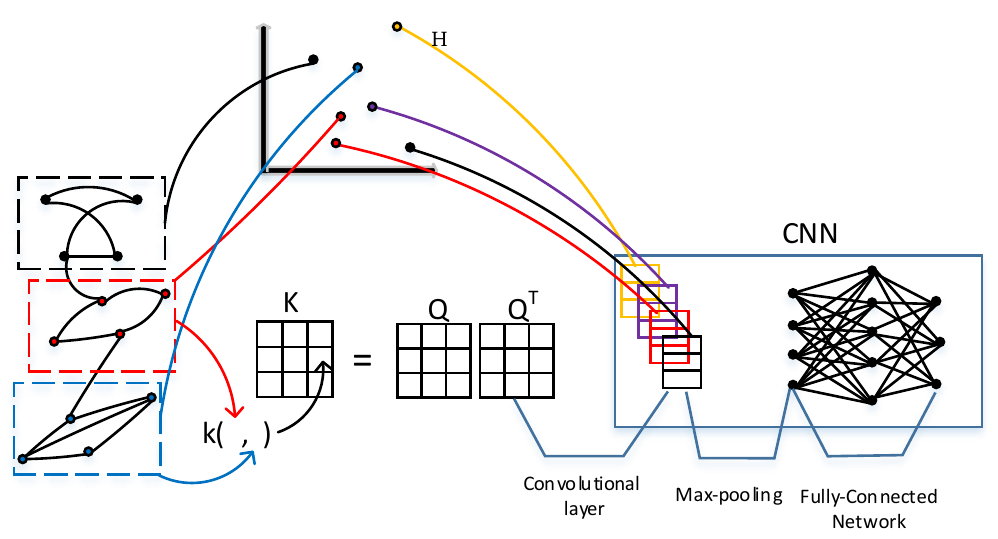}
  \caption{Overview of our Kernel Graph CNN approach.}
  \label{fig:overview}
\end{figure}

\subsection{Patch Extraction and Normalization}
Many types of real-world data are regular grids, and can thus be decomposed into units that are inherently ordered along spatial dimensions. This makes the task of patch extraction easy, and normalization unnecessary. For example, in computer vision (2D), meaningful patches are given by instantiating a rectangle window over the image. Furthermore, for all images, pixels are uniquely ordered along width and height, so there is a correspondence between the pixels in each patch, given by the spatial coordinates of the pixels. This removes the need for normalization. Likewise, in NLP, words in sentences are uniquely ordered from left to right, and a 1D window applied over text provides again natural regions. However, graphs do not exhibit such an underlying grid-like structure. They are irregular objects for which there exist no canonical ordering of the elementary units (nodes). Hence, generating patches from graphs, and normalizing them so that they are comparable and combinable, is a very challenging problem. To address these challenges, our approach leverages \textit{community detection} and \textit{graph kernels}.

\noindent\textbf{Patch extraction with community detection.}
There is a large variety of approaches for sampling from graphs. We can extract subgraphs for all vertices (which may be computationally intractable for large graphs) or for only a subset of them, such as the most central ones according to some metric. Furthermore, subgraphs may contain only the hop-1 neighborhood of a root vertex, or vertices that are further away from it. They may also be walks passing through the root vertex. A more natural way is to capitalize on \textit{community detection} algorithms \cite{fortunato2016community}, as the clusters correspond to meaningful graph partitions. Indeed, a community typically corresponds to a set of vertices that highly interact with each other, as expressed by the number and weight of the edges between them, compared to the other vertices in the graph. In this paper, we employ the Louvain clustering algorithm, which extracts non-overlapping communities of various sizes from a given graph \cite{blondel2008fast}. This multilevel algorithm aggregates each node with one of its neighbors such that the gain in modularity is maximized. Then, the groupings obtained at the first step are turned into nodes, yielding a new graph. The process iterates until a peak in modularity is attained and no more change occurs. Note that since our goal here is only to sample relevant local neighborhoods from the graph, we could have used any other state-of-the-art community detection algorithm. We opted for Louvain as it is very fast and scalable.

\noindent\textbf{Patch normalization with graph kernels.}
After extracting the subgraphs (communities) from a given input graph, standardization is necessary before being able to pass them to a CNN. We can define this step as that of \textit{patch normalization}. To this purpose, we leverage graph kernels, as described next. Note that since the steps below do not depend on the way the subgraphs were obtained, we use the term \textit{subgraph} (or \textit{patch}) rather than \textit{community} in what follows, to highlight the generality of our approach.

Let $\mathcal{G} = \{ G_1, G_2, \ldots, G_N \}$ be the collection of input graphs. Let $\mathcal{S}_1, \mathcal{S}_2, \ldots, \mathcal{S}_N$ be the sets of subgraphs extracted from graphs $G_1, G_2, \ldots, G_N$ respectively. Since the number of subgraphs extracted from each graph may depend on the graph (like in our case with the Louvain community detection algorithm), these sets vary in size.

Furthermore, let $S_i^j$ be the $j^{th}$ element of $\mathcal{S}_i$ (i.e., the $j^{th}$ subgraph extracted from $G_i$), and $P_i$ be the size of $\mathcal{S}_i$ (i.e., the total number of subgraphs extracted from $G_i$). Let then $\mathcal{S} = \{S_i^j : i \in \{1,2,\ldots,N\}, j \in \{1,2,\ldots,P_i \} \}$ be the set of subgraphs extracted from all the graphs in the collection, and $P$ its cardinality.
Let finally $K \in \mathbb{R}^{P \times P}$ be the symmetric positive semidefinite kernel matrix constructed from $\mathcal{S}$ using a graph kernel $k$.
Since the total number $P$ of subgraphs for all the graphs in the collection  is very large, populating the full kernel matrix $K$ and factorizing it to obtain low-dimensional representations of the subgraphs is $\mathcal{O}(P^3)$. Fortunately, the Nystr{\"o}m method \cite{williams2000using} allows us to obtain $Q \in \mathbb{R}^{P \times p}$ (with $p \ll P$) such that $K \approx QQ^\top$ at the reduced cost of $\mathcal{O}(p^2P)$, by using only a small subset of $p$ columns (or rows) of the kernel matrix. The rows of $Q$ are low-dimensional representations of the subgraphs and serve as our normalized patches.
 
\subsection{Graph processing}\label{sub:graph_proc}
\noindent\textbf{1D Convolution.}
To process a given input graph, many filters are convolved with the normalized representations of the patches contained in the graph. For example, for a given filter $w \in \mathbb{R}^p$, a feature $c_i$ is generated from the $j^{th}$ patch of graph $G_i$ $z_i^j$ as:
\begin{equation*}
	c_j = \sigma(w^\top z_i^j)
\end{equation*}
where $\sigma$ is an activation function. In this study, we used the identity function $\sigma(c) = c$, as we observed no difference in results compared to nonlinear activations. Therefore, when applied to a patch $z_i^j$, the convolution operation corresponds to the inner product $\langle w, z_i^j \rangle$.
We will show next that any filter $w$ with $||w|| < \infty$ learned by our network belongs to the Reproducing Kernel Hilbert Space (RKHS) $\mathcal{H}$ of the employed graph kernel $k$.
\begin{theorem}
The filters live in the RKHS of the kernel $k$ that was used to normalize the patches.
\end{theorem}
\begin{proof}
Given two subgraphs $S_i^j$ and $S_{i'}^{j'}$ extracted from $G_i$ and $G_i'$ and their associated normalized patches $z_i^j$ and $z_{i'}^{j'}$, it holds that:
\begin{equation*}
	\langle z_i^j, z_{i'}^{j'} \rangle = k(S_i^j, S_{i'}^{j'}) = \langle \phi(S_i^j), \phi(S_{i'}^{j'}) \rangle_{\mathcal{H}}
\end{equation*}
Let $\mathcal{Z} = \{z_i^j : i \in \{1,2,\ldots,N\}, j \in \{1,2,\ldots,P_i \} \}$ be the set containing all patches of the input graphs.
Then, $\mathrm{Span}(\mathcal{Z})$ is either the space of all vectors in $\mathbb{R}^P$ if the rank of the kernel matrix is $P$ or the space of all vectors in $\mathbb{R}^P$ whose last $t$ components are zero if the rank of the kernel matrix is $P-t$ where $t > 0$.
Then, given a patch $z_i^j$, vector $w$ is contained in $\mathrm{Span}(\mathcal{Z})$, hence:
\begin{equation*}
\begin{split}
	\sigma(w^\top z_i^j) = \langle w, z_i^j \rangle &= \langle \sum_{i'=1}^N \sum_{j'=1}^{P_i} a_{i'}^{j'} z_{i'}^{j'}, z_i^j \rangle \\
	& = \sum_{i'=1}^N \sum_{j'=1}^{P_i} a_{i'}^{j'} \langle z_{i'}^{j'}, z_i^j \rangle = \sum_{i'=1}^N \sum_{j'=1}^{P_i} a_{i'}^{j'} k(S_{i'}^{j'}, S_i^j)
\end{split}
\end{equation*}
which shows that the filters live in the RKHS associated to graph kernel $k$. For other smooth activation functions, one can also show that the filters will be contained in the corresponding RKHS of the kernel function \cite{zhang17f}.
\end{proof}
Note that the proposed approach can be thought of as a CNN that works directly on graphs. In computer vision, convolution corresponds to the element-wise multiplication between part of an image and a filter followed by summation.
Convolution can thus be viewed as an inner-product where the output is a single feature.
In our setting, convolution corresponds to the inner-product between part of a graph (\ie a patch) and a filter (\ie a graph). Such an inner-product is implicitly computed using a graph kernel, and the output is also a single feature.

By convolving $w$ with all the normalized patches of the graph, the following feature map is produced:
\begin{equation*}
  c = [c_1, c_2, \ldots, c_{P_{max}}]^\top
\end{equation*}
where $P_{max} = \max(P_{i}: i \in \{1,2,\ldots,N\})$ is the largest number of patches extracted from any given graph in the collection. For graphs featuring less than $P_{max}$ patches, zero-padding is employed.

Note that this approach is similar to concatenating all the vector representations of the patches contained in a given graph (padding if necessary), thus obtaining a single vector representation of the graph, and sliding over it a unidimensional filter of size the length of a single patch vector, without overspanning patches (i.e., with stride equal to filter size).

\noindent\textbf{Pooling.} We then apply a max-pooling operation over the feature map, thus retaining only the maximum value of $c$, $\max(c_1,c_2,\ldots,c_{P_{max}})$, as the signal associated with $w$. The intuition is that some subgraphs of a graph are good indicators of the class the graph belongs to, and that this information will be picked up by the max-pooling operation.

\subsection{Processing new graphs}
When provided with a never-seen graph (at test time), we first sample subgraphs from it (here, via community detection), and then project them to the feature space of the subgraphs in the training set. Given a new subgraph $S^j$, its projection can be computed as $z^j = Q^\dagger v$ where $Q^\dagger \in \mathbb{R}^{p \times P} $ is the pseudoinverse of $Q \in \mathbb{R}^{P \times p}$ and $v \in \mathbb{R}^P$ is the vector containing the kernel value between $S^j$ and all $P$ subgraphs in the training set (those contained in set $\mathcal{S}$).
The dimensionality $p$ of the emerging vector is the same as that of the normalized patches in the training set. Thus, this vector can be convolved with the filters of the CNN as previously described.

\subsection{Channels}\label{sub:channels}
Rather than selecting one graph kernel in particular to normalize the patches, several kernels can be jointly used. The different representations provided by each kernel can then be passed to the CNN through different channels, or \textit{depth} dimensions. Intuitively, this can be very beneficial, as each kernel might capture different, complementary aspects of similarity between subgraphs. We experimented with the following popular kernels:\\
$\bullet$ \textbf{Shortest path kernel (SP)} \cite{borgwardt2005shortest}: to compute the similarity between two graphs, this kernel counts how many pairs of shortest paths have the same source and sink labels, and identical length, in the two graphs. The runtime complexity for a pair of graphs featuring ${n_1}$ and ${n_2}$ nodes is $\mathcal{O}({n_1}^2 {n_2}^2)$.\\ 
$\bullet$ \textbf{Weisfeiler-Lehman subtree kernel (WL)} \cite{shervashidze2011weisfeiler}: for a certain number $h$ of iterations, this kernel performs an exact matching between the compressed multiset labels of the two graphs, while at each iteration it updates these labels.
It requires $\mathcal{O}(hm)$ time for a pair of graphs with $m$ edges.

This gave us two single channel models (KCNN SP, KCNN WL), and one model with two channels (KCNN SP+WL).

\section{Experimental Setup}\label{sec:experiments}

\subsection{Synthetic Dataset}\label{sub:synthetic_dataset}
\textbf{Dataset.} As previously mentioned, the intuition is that our proposed KCNN model is particularly well suited for settings where some regions in the graphs are highly discriminative of the class the graph belongs to.
To empirically verify this claim, we created a dataset featuring $1000$ synthetic graphs generated as follows. First, we generate an Erdos-R{\'e}nyi graph with number of vertices sampled from $\mathbb{Z}\cap\big[100,200\big]$ with uniform probability, and edge probability equal to $0.1$. We then add to the graph either a $10$-clique or a $10$-star graph by connecting the vertices with probability $0.1$. The first class of the dataset is made of the graphs containing a $10$-clique, while the second class features the graphs containing a $10$-star subgraph. The two classes are of equal size ($500$ graphs each).

\noindent\textbf{Baselines.} We compared our model against the shortest-path kernel (SP) \cite{borgwardt2005shortest}, the Weisfeiler-Lehman subtree kernel (WL) \cite{shervashidze2011weisfeiler}, and the graphlet kernel (GR) \cite{shervashidze2009efficient}.

\noindent\textbf{Configuration.} We performed $10$-fold cross-validation. The $C$ parameter of the SVM (for all graph kernels) and the number of iterations (for the WL kernel baseline) were optimized on a 90/10 split of the training set of each fold. For the graphlet kernel, we sampled $1000$ graphlets of size up to $6$ from each graph. For our proposed KCNN, we used an architecture with one convolution-pooling block followed by a fully connected layer with $128$ units. The \texttt{ReLU} activation was used, and regularization was ensured with dropout ($0.5$ rate). A final \texttt{softmax} layer was added to complete the architecture. The dimensionality of the normalized patches (number of columns of $Q$) was set to $p=100$, and we used $256$ filters (of size $p$, as explained in subsection \ref{sub:graph_proc}). Batch size was set to $64$, and the number of epochs and learning rate were optimized by performing 10-fold cross-validation on the training set of each fold. All experiments were run on a single machine consisting of a $3.4$ GHz Intel Core i$7$ CPU with $16$ GB of RAM and an NVidia GeForce Titan Xp GPU.

\noindent\textbf{Results.} We report in Table~\ref{tab:results_synthetic} average prediction accuracies of our three models in comparison to the baselines. Results validated the hypothesis that our proposed model (KCNN) can identify those areas in the graphs that are most predictive of the class labels, as its three variants achieved accuracies greater than $98\%$. 
Conversely, the baseline kernels failed to discriminate between the two categories.
Hence, it is clear that in such settings, our model is more effective than existing methods.
\begin{table}[t]
\caption{Classification accuracy of state-of-the-art graph kernels: shortest path (SP), graphlet (GR), and Weisfeiler-Lehman subtree (WL); and the single and multichannel variants of our approach (KCNN), on the synthetic dataset.}
\vspace{-.2cm}
\centering
\scriptsize
\begin{tabular}{ |c|c|c|c|c|c| } 
 \hline
 SP & GR & WL & KCNN SP & KCNN WL & KCNN SP+WL \\
  \hline
 75.47 & 69.34 & 65.88 & 98.20 & 97.25 & \textbf{98.40} \\ 
 \hline
\end{tabular}
\label{tab:results_synthetic}
\vspace{-.5cm}
\end{table}

\subsection{Real-World Datasets}\label{sec:datasets}
\noindent\textbf{Datasets.} We also evaluated the performance of our approach on five bioinformatics (ENZYMES, NCI1, PROTEINS, PTC-MR, D\&D) and five social network datasets (IMDB-BINARY, IMDB-MULTI, REDDIT-BINARY, REDDIT-MULTI-5K, COLLAB)\footnote{The datasets, further references and statistics are available at \url{https://ls11-www.cs.tu-dortmund.de/staff/morris/graphkerneldatasets}}.
Notice that the bioinformatics datasets are labeled (labels on vertices), while the social interaction datasets are not.

\noindent\textbf{Baselines.} We evaluated our model in comparison with the shortest-path kernel (SP) \cite{borgwardt2005shortest}, the random walk kernel (RW) \cite{vishwanathan2010graph}, the graphlet kernel (GR) \cite{shervashidze2009efficient}, the Weisfeiler-Lehman subtree kernel (WL) \cite{shervashidze2011weisfeiler}, the best kernel from the deep graph kernel framework (Deep Graph Kernels) \cite{yanardag2015deep}, and a recently proposed graph CNN (PSCN $k=10$) \cite{niepert2016learning}. Since the experimental setup is the same, we report the results of \cite{yanardag2015deep} and \cite{niepert2016learning}.

\noindent\textbf{Configuration.} Same as~\ref{sub:synthetic_dataset} above.

\noindent\textbf{Results.} The $10$-fold cross-validation average test set accuracy of our approach and the baselines is reported in Table~\ref{tab:results}. Our approach outperforms all baselines on $7$ out of the $10$ datasets.
In some cases, the gains in accuracy over the best performing competitors are considerable.
For instance, on the IMDB-MULTI, COLLAB, and D\&D datasets, we offer respective \textit{absolute} improvements of $2.23\%$, $2.33\%$, and $2.56\%$ in accuracy over the best competitor, the state-of-the-art graph CNN (PSCN $k=10$).
Finally, it should be noted that on the IMDB-MULTI dataset, every variant of our architecture outperforms \textit{all} baselines.

\begin{table}[t]
\caption{$10$-fold cross validation average classification accuracy ($\pm$ standard deviation) of the proposed models and the baselines on the bioinformatics (top) and social network (bottom) datasets. Best performance per dataset in \textbf{bold}, among the variants of our Kernel CNN model \underline{underlined}.}
\vspace{-.7cm}
\begin{center}
\def\arraystretch{1.1}
\sisetup{
detect-all,
table-text-alignment = center,
table-number-alignment = center,
table-figures-integer = 2,
table-figures-decimal = 2,
}
\resizebox{\textwidth}{!} {
\begin{tabular}{|l|Hccccc|} \hline
\multirow{2}{*}{\backslashbox{Method}{Dataset}} & \multirow{2}{*}{MUTAG} & \multirow{2}{*}{ENZYMES} & \multirow{2}{*}{NCI1} & \multirow{2}{*}{PROTEINS} & \multirow{2}{*}{PTC-MR} & \multirow{2}{*}{D\&D} \\
& & & & & & \\ \hline \hline
SP & 85.79 ($\pm$ 2.51) & 40.10 ($\pm$ 1.50) & 73.00 ($\pm$ 0.51) & 75.07 ($\pm$ 0.54) & 58.24 ($\pm$ 2.44) & $>$ 3 days  \\
GR & 81.66 ($\pm$ 2.11) & 26.61 ($\pm$ 0.99) & 62.28 ($\pm$ 0.29) & 71.67 ($\pm$ 0.55) & 57.26 ($\pm$ 1.41) & 78.45 ($\pm$ 0.26) \\ 
RW & 83.72 ($\pm$ 1.50) & 24.16 ($\pm$ 1.64) & $>$ 3 days & 74.22 ($\pm$ 0.42) & 57.85 ($\pm$ 1.30) & $>$ 3 days \\ 
WL & 80.72 ($\pm$ 3.00) & 53.15 ($\pm$ 1.14)  & 80.13 ($\pm$ 0.50) & 72.92 ($\pm$ 0.56) & 56.97 ($\pm$ 2.01) & 77.95 ($\pm$ 0.70) \\ 
Deep Kernels & 87.44 ($\pm$ 2.72) & \textbf{53.43} ($\pm$ 0.91) & \textbf{80.31} ($\pm$ 0.46) & 75.68 ($\pm$ 0.54) & 60.08 ($\pm$ 2.55) & NA \\
PSCN $k=10$ & \textbf{88.95} ($\pm$ 4.37) & NA & 76.34 ($\pm$ 1.68) & 75.00 ($\pm$ 2.51) & 62.29 ($\pm$ 5.68) & 76.27 ($\pm$ 2.64) \\  \hline \hline
KCNN SP & \underline{84.17} ($\pm$ 0.42) & 46.35 ($\pm$ 0.39) & 75.70 ($\pm$ 0.31) & 74.27 ($\pm$ 0.22) & \textbf{\underline{62.94}} ($\pm$ 1.69) & 76.63 ($\pm$ 0.09) \\
KCNN WL & 84.09 ($\pm$ 0.12) & 43.08 ($\pm$ 0.68) & 75.83 ($\pm$ 0.25) & \textbf{\underline{75.76}} ($\pm$ 0.28) & 61.52 ($\pm$ 1.41) & 75.80 ($\pm$ 0.07) \\ \hline \hline
KCNN SP + WL & 83.71 ($\pm$ 0.16) & \underline{48.12} ($\pm$ 0.23) & \underline{77.21} ($\pm$ 0.22) & 73.79 ($\pm$ 0.29) & 62.05 ($\pm$ 1.41) & \textbf{\underline{78.83}} ($\pm$ 0.29) \\ \hline
\end{tabular}
}
\vspace{.1cm}
\\
\resizebox{\textwidth}{!} {
\begin{tabular}{|l|ccccc|} \hline
\multirow{2}{*}{\backslashbox{Method}{Dataset}} & IMDB & IMDB & REDDIT & REDDIT & \multirow{2}{*}{COLLAB}\\
& BINARY & MULTI & BINARY & MULTI-5K & \\ \hline \hline
GR & 65.87 ($\pm$ 0.98) & 43.89 ($\pm$ 0.38) & 77.34 ($\pm$ 0.18) & 41.01 ($\pm$ 0.17) & 72.84 ($\pm$ 0.28) \\
Deep GR & 66.96 ($\pm$ 0.56) & 44.55 ($\pm$ 0.52) & 78.04 ($\pm$ 0.39) & 41.27 ($\pm$ 0.18) & 73.09 ($\pm$ 0.25) \\ 
PSCN $k=10$ & 71.00 ($\pm$ 2.29) & 45.23 ($\pm$ 2.84) & \textbf{86.30} ($\pm$ 1.58) & 49.10 ($\pm$ 0.70) & 72.60 ($\pm$ 2.15) \\ \hline \hline
KCNN SP & 69.60 ($\pm$ 0.44) & 45.99 ($\pm$ 0.23) & 77.23 ($\pm$ 0.15) & 44.86 ($\pm$ 0.24) & 70.78 ($\pm$ 0.12) \\
KCNN WL & 70.46 ($\pm$ 0.45) & 46.44 ($\pm$ 0.24) & \underline{81.85} ($\pm$ 0.12) & \textbf{\underline{50.04}} ($\pm$ 0.19) & \textbf{\underline{74.93}} ($\pm$ 0.14) \\ \hline \hline
KCNN SP + WL & \textbf{\underline{71.45}} ($\pm$ 0.15) & \textbf{\underline{47.46}} ($\pm$ 0.21) & 78.35 ($\pm$ 0.11) & 44.63 ($\pm$ 0.18) & 74.12 ($\pm$ 0.17) \\\hline
\end{tabular}
}
\label{tab:results}
\end{center}
\vspace{-.5cm}
\end{table}

\noindent\textbf{Interpretation.} Overall, our Kernel CNN model reaches better performance than the classical graph kernels (SP, GR, RW, and WL), showing that the ability of CNNs to learn their own features during training is superior to disjoint feature computation and learning. It is true that our approach also comprises two disjoint steps. However, the first step is only a \textit{data preprocessing} step, where we extract neighborhoods from the graphs, and normalize them with graph kernels. The features used for classification are then learned \textit{during training} by our neural architecture, unlike the GK + SVM approach, where the features, given by the kernel matrix, are computed in advance, independently from the downstream task.

Our two single-channel architectures perform comparably on the bioinformatics datasets, while the KCNN WL variant was superior on the social network datasets. On the REDDIT-BINARY, REDDIT-MULTI-5K and COLLAB datasets, KCNN WL also outperforms the multichannel architecture, with quite wide margins. The multi-channel architecture (KCNN SP + WL) leads to better results on 5 out of the 10 datasets, showing that capturing subgraph similarity from a variety of angles sometimes helps.

\begin{table}[t]
\caption{$10$-fold cross validation runtime of proposed models on the $10$ real-world graph classification datasets.}
\vspace{-.7cm}
\begin{center}
\def\arraystretch{1.1}
\resizebox{\textwidth}{!} {
\begin{tabular}{lHcccccccccc} \hline
\multirow{2}{*}{} & \multirow{2}{*}{MUTAG} & \multirow{2}{*}{ENZYMES} & \multirow{2}{*}{NCI1} & \multirow{2}{*}{PROTEINS} & \multirow{2}{*}{PTC-MR} & \multirow{2}{*}{D\&D} & IMDB & IMDB & REDDIT & REDDIT & \multirow{2}{*}{COLLAB} \\ 
& & & & & & & BINARY & MULTI & BINARY & MULTI-5K\\ \hline
KCNN SP & 14'' & 28'' &  4' 26'' & 42'' & 22'' & 54'' & 36'' & 1' 41'' & 5' 29'' & 15' 2'' & 7' 2'' \\
KCNN WL & 15'' & 53'' & 4' 54'' & 48'' & 22'' & 1' 33'' & 41'' & 58'' & 5' 22'' & 14' 23'' & 8' 58'' \\
KCNN SP+WL & 16'' & 1' 13'' & 5' 1'' & 53'' & 25'' & 1' 46'' & 45'' & 1' 44'' & 9' 57'' & 24' 28'' & 10' 24'' \\ \hline
\end{tabular}
}
\label{tab:runtimes}
\end{center}
\vspace{-.8cm}
\end{table}

\noindent\textbf{Runtimes.} We also report the time cost of our three models in Table \ref{tab:runtimes}.
Runtime includes all steps of the process: patch extraction, path normalization, and $10$-fold cross validation procedure.
We can see that the computational complexity of the proposed models is not high. Our most computationally intensive model (KCNN SP+WL) takes less than 25 minutes to perform the full 10-fold cross validation procedure on the largest dataset (REDDIT-MULTI-5K).
Moreover, in most cases, the running times are lower or comparable to the ones of the state-of-the-art Graph CNN and Deep Graph Kernels models \cite{niepert2016learning,yanardag2015deep}.

\section{Conclusion}\label{sec:conclusion}
In this paper, we proposed a method that combines graph kernels with CNNs to learn graph representations and to perform graph classification. Our Kernel Graph CNN model (KCNN) outperforms 6 state-of-the-art graph kernels and graph CNN baselines on 7 datasets out of 10.

\bibliographystyle{splncs04}
\bibliography{biblio}

\begin{thebibliography}{10}
\providecommand{\url}[1]{\texttt{#1}}
\providecommand{\urlprefix}{URL }
\providecommand{\doi}[1]{https://doi.org/#1}

\bibitem{blondel2008fast}
Blondel, V.D., Guillaume, J.L., Lambiotte, R., Lefebvre, E.: Fast unfolding of
  communities in large networks. JSTAT  \textbf{2008}(10),  1--12 (2008)

\bibitem{borgwardt2005shortest}
Borgwardt, K.M., Kriegel, H.: Shortest-path kernels on graphs. In: ICDM. pp.
  74--81 (2005)

\bibitem{bruna2014}
Bruna, J., Zaremba, W., Szlam, A., LeCun, Y.: {Spectral Networks and Locally
  connected networks on Graphs}. In: ICLR (2014)

\bibitem{defferrard2016convolutional}
Defferrard, M., Bresson, X., Vandergheynst, P.: {Convolutional Neural Networks
  on Graphs with Fast Localized Spectral Filtering}. In: NIPS. pp. 3837--3845
  (2016)

\bibitem{fortunato2016community}
Fortunato, S., Hric, D.: Community detection in networks: A user guide. Physics
  Reports  \textbf{659},  1--44 (2016)

\bibitem{horvath2004cyclic}
Horv{\'a}th, T., G{\"a}rtner, T., Wrobel, S.: {Cyclic Pattern Kernels for
  Predictive Graph Mining}. In: KDD. pp. 158--167 (2004)

\bibitem{johansson2014global}
Johansson, F., Jethava, V., Dubhashi, D., Bhattacharyya, C.: Global graph
  kernels using geometric embeddings. In: ICML. pp. 694--702 (2014)

\bibitem{kipf2016semi}
Kipf, T.N., Welling, M.: {Semi-Supervised Classification with Graph
  Convolutional Networks}. In: ICLR (2017)

\bibitem{kondor2016multiscale}
Kondor, R., Pan, H.: {The Multiscale Laplacian Graph Kernel}. In: NIPS. pp.
  2982--2990 (2016)

\bibitem{niepert2016learning}
Niepert, M., Ahmed, M., Kutzkov, K.: {Learning Convolutional Neural Networks
  for Graphs}. In: ICML (2016)

\bibitem{nikolentzos2017matching}
Nikolentzos, G., Meladianos, P., Vazirgiannis, M.: {Matching Node Embeddings
  for Graph Similarity}. In: AAAI. pp. 2429--2435 (2017)

\bibitem{shervashidze2011weisfeiler}
Shervashidze, N., Schweitzer, P., Van~Leeuwen, E.J., Mehlhorn, K., Borgwardt,
  K.M.: {Weisfeiler-Lehman Graph Kernels}. JMLR  \textbf{12},  2539--2561
  (2011)

\bibitem{shervashidze2009efficient}
Shervashidze, N., Vishwanathan, S., Petri, T., Mehlhorn, K., Borgwardt, K.M.:
  Efficient graphlet kernels for large graph comparison. In: AISTATS. pp.
  488--495 (2009)

\bibitem{tixier2017classifying}
Tixier, A., Nikolentzos, G., Meladianos, P., Vazirgiannis, M.: {Classifying
  Graphs as Images with Convolutional Neural Networks}. arXiv:1708.02218
  (2017)

\bibitem{vialatte2016generalizing}
Vialatte, J.C., Gripon, V., Mercier, G.: {Generalizing the Convolution Operator
  to extend CNNs to Irregular Domains}. arXiv preprint arXiv:1606.01166  (2016)

\bibitem{vishwanathan2010graph}
Vishwanathan, S.V.N., Schraudolph, N.N., Kondor, R., Borgwardt, K.M.: {Graph
  Kernels}. JMLR  \textbf{11},  1201--1242 (2010)

\bibitem{williams2000using}
Williams, C.K., Seeger, M.: {Using the Nystr{\"o}m Method to Speed Up Kernel
  Machines}. In: NIPS. pp. 661--667 (2000)

\bibitem{yanardag2015structural}
Yanardag, P., Vishwanathan, S.: {A Structural Smoothing Framework For Robust
  Graph Comparison}. In: NIPS. pp. 2125--2133 (2015)

\bibitem{yanardag2015deep}
Yanardag, P., Vishwanathan, S.: {Deep Graph Kernels}. In: KDD. pp. 1365--1374
  (2015)

\bibitem{zhang17f}
Zhang, Y., Liang, P., Wainwright, M.J.: Convexified convolutional neural
  networks. In: ICML. pp. 4044--4053 (2017)

\end{thebibliography}

\end{document}